\documentclass{article}

\PassOptionsToPackage{dvipsnames,table}{xcolor}
\usepackage{iclr2027_conference,times}
\usepackage{amsmath,amsfonts,bm}

\def\eqref#1{equation~\ref{#1}}

\def\1{\bm{1}}

\DeclareMathAlphabet{\mathsfit}{\encodingdefault}{\sfdefault}{m}{sl}
\SetMathAlphabet{\mathsfit}{bold}{\encodingdefault}{\sfdefault}{bx}{n}

\newcommand{\E}{\mathbb{E}}

\usepackage{amsmath,amssymb,mathtools}
\usepackage{amsthm}
\usepackage{booktabs}
\usepackage{multirow}
\usepackage{graphicx}
\usepackage{float}
\usepackage{xurl}
\usepackage{enumitem}
\setlist{topsep=2pt,itemsep=1pt,parsep=0pt}
\usepackage{pifont}        
\usepackage{tcolorbox}     
\tcbuselibrary{breakable}

\DeclareUnicodeCharacter{266B}{}

\definecolor{BrightTeal}{HTML}{45B8AC}
\definecolor{AquaMint}{HTML}{00B89F}
\definecolor{darkred}{RGB}{178,34,34}
\definecolor{questionbg}{RGB}{235,249,255}     
\definecolor{questionframe}{RGB}{135,190,230}  
\definecolor{insightbg}{RGB}{255,250,240}      
\definecolor{insightframe}{RGB}{200,160,100}   

\usepackage[
  colorlinks=true,
  linkcolor=purple,
  citecolor=Violet,
  urlcolor=AquaMint
]{hyperref}

\newcommand{\ours}{\textsc{\textbf{Dia}}}
\newcommand{\oursfull}{\textbf{D}isagreement-triggered \textbf{I}ndependent \textbf{A}rbitration}
\newcommand{\tb}[1]{\textbf{#1}}

\newcommand{\pcref}[1]{\textcolor{purple}{\ref{#1}}}
\definecolor{paleaccent}{HTML}{F3EEF8}   
\newcommand{\obs}[1]{\textcolor{purple}{\textbf{Observation~\ding{#1}.}}}
\newcommand{\takeaway}[1]{\textcolor{purple}{\textbf{Takeaway.}} \textbf{#1}}
\newtcolorbox{corebox}{
  colback=paleaccent, colframe=purple!70!black,
  boxrule=0.45pt, arc=1.5pt,
  left=5pt, right=5pt, top=3pt, bottom=3pt,
  breakable
}

\theoremstyle{plain}
\newtheorem{proposition}{Proposition}
\newtheorem{propositionf}{Proposition}[section]   
\theoremstyle{remark}
\newtheorem{remark}{Remark}
\newcommand{\ind}[1]{\mathbf{1}\!\left[#1\right]}

\newtcolorbox{questionbox}{
  colback=questionbg, colframe=questionframe,
  boxrule=0.4pt, arc=1pt,
  left=6pt, right=6pt, top=4pt, bottom=4pt,
  fonttitle=\bfseries, coltitle=black,
  breakable
}

\newtcolorbox{insightbox}{
  colback=insightbg, colframe=insightframe,
  boxrule=0.4pt, arc=1pt,
  left=6pt, right=6pt, top=4pt, bottom=4pt,
  breakable
}

\title{Not All Test-Time Compute Is Evidence:\\Model-Specific Bias in LLM Reasoning}

\author{
Ming Shen$^{1,2}$ \quad Tianjun Yao$^{2}$ \\
\textnormal{$^{1}$Shanghai Jiao Tong University}\\
\textnormal{$^{2}$Shenzhen Loop Area Institute (SLAI)}
}

\iclrfinalcopy

\begin{document}

\maketitle
\lhead{}

\begin{abstract}
Test-time scaling methods implicitly assume that additional computation
from the same model provides independent evidence. We show that this
assumption fails within a single reasoning model: errors are correlated
across repeated samples, across reasoning depth, and in the model's own
self-evaluation, so more compute from the same model is not more evidence.
We introduce \ours{} (\oursfull), which treats disagreement between an
intermediate-stage and the final-stage prediction as a per-question signal
of reasoning instability and resolves it by asking a small panel of
heterogeneous models to solve the question independently, without seeing
either candidate, and aggregating their answers directly. Across seven
benchmarks and two base models, \ours{} outperforms full chain-of-thought
in 12 of 14 settings, including a \textbf{7.00\%} gain on MuSiQue, while
activating on only ${\sim}$\textbf{17\%} of examples and using up to
\textbf{47.2\% fewer tokens} than a multi-agent baseline. Ablations show
that both independent solving and solver heterogeneity are necessary for
these gains.
\end{abstract}


\section{Introduction}
\label{sec:introduction}

Test-time scaling has become a central strategy for improving reasoning in
large language models: chain-of-thought
prompting~\citep{wei2022chain}, repeated sampling with majority
vote~\citep{wang2023selfconsistency}, adaptive early
exit~\citep{ICLR2026_8df90a14, fu2024certaindex, dai2025sgrpo}, and
multi-agent
reasoning~\citep{wu2023autogen, li2023camel, li2024moreagents} all convert additional
inference-time computation into accuracy
gains~\citep{snell2025scaling, deepseekai2025deepseekr1}.
Test-time scaling is particularly effective when additional computation contributes
evidence that is sufficiently complementary to the model's current
belief, allowing errors to be corrected rather than repeatedly
reinforced.
However, computation from the same model can remain shaped by persistent
\textcolor{purple}{\textbf{model-specific error bias}}, limiting the
diversity of the evidence it provides.
We observe three manifestations of this phenomenon across generation,
reasoning, and self-evaluation
(Figure~\pcref{fig:intro-composite}).

\begin{figure*}[t]
    \centering
    \includegraphics[width=\textwidth]{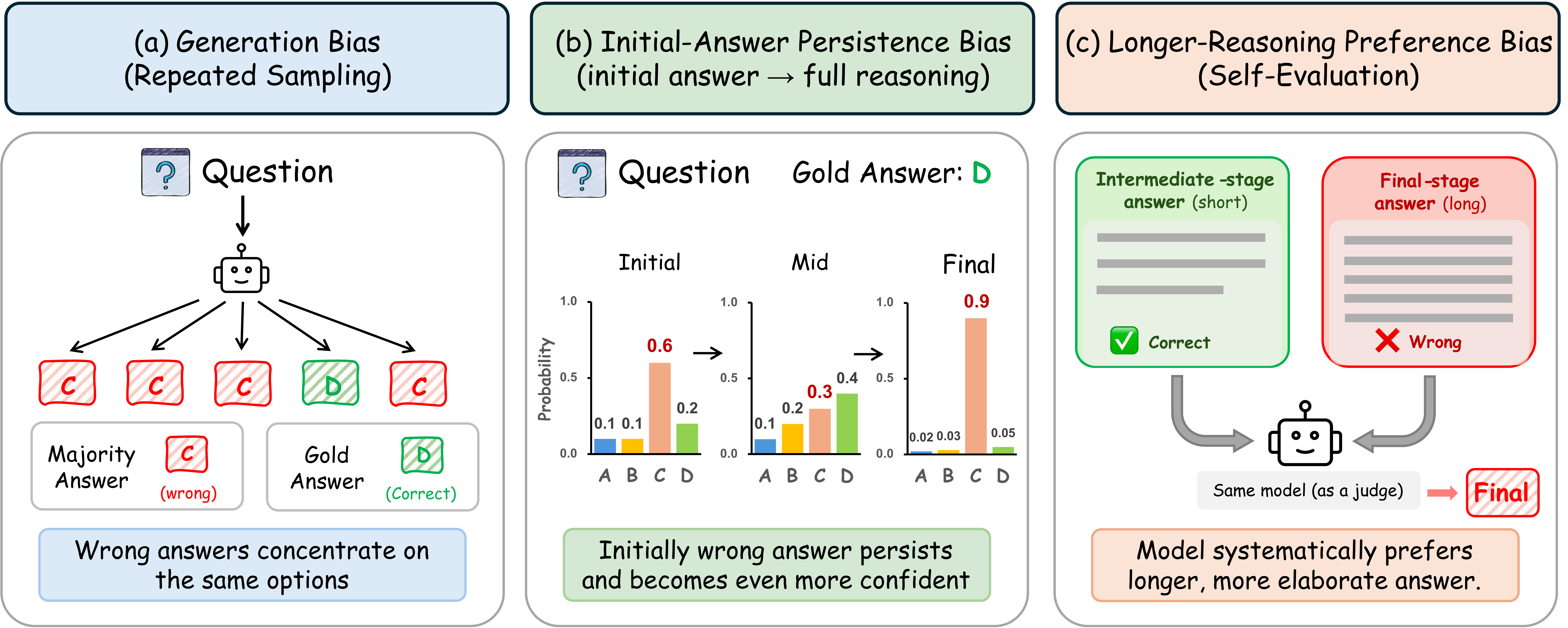}
    \caption{Three manifestations of
    \textcolor{purple}{\textbf{model-specific error bias}} within a
    single reasoning model.
    \tb{(a)}~When repeated sampling produces an incorrect majority
    answer, erroneous samples concentrate on the same wrong answer
    rather than spreading across alternatives.
    \tb{(b)}~When a model's initial answer is incorrect, reasoning
    frequently preserves the same wrong answer rather than correcting
    it.
    \tb{(c)}~When answers at different reasoning stages disagree, the
    model's self-evaluation systematically favors the later, more
    elaborate outcome, even when the earlier answer is correct.
    Detailed analyses are in Section~\pcref{sec:motivation}.}
    \label{fig:intro-composite}
\end{figure*}

\tb{(a)~Generation bias under repeated sampling.}
Repeated sampling can diversify reasoning traces without necessarily
diversifying their failures.
When self-consistency produces an incorrect majority answer, erroneous
samples frequently concentrate on the same dominant wrong answer rather
than spreading across alternatives, indicating a persistent generation
bias rather than independent error modes
(Section~\pcref{sec:motivation-concentration}).

\tb{(b)~Initial-answer persistence bias.}
Reasoning does not necessarily provide evidence independent of a model's
initial judgment.
When a model is first asked to answer directly and then asked to reason
before answering, its post-reasoning prediction can remain strongly
associated with the initial answer; in particular, an initially
incorrect prediction may persist through subsequent reasoning rather
than being corrected (Section~\pcref{sec:motivation-persistence}).
Prior work similarly suggests that reasoning models can form strong
preliminary answer
preferences~\citep{dang2025internal}, that extended reasoning can lead
to overthinking~\citep{zhou2026whenmorethinkinghurts}, and that early
errors may propagate through subsequent
reasoning~\citep{zhang2023language}.

\tb{(c)~Longer-reasoning preference bias.}
A related bias appears when the model is asked to evaluate competing
answers produced at different stages of its own reasoning.
Rather than judging the candidates symmetrically, the model strongly
favors the later, more elaborate reasoning outcome, even when the
earlier answer is correct.
In our experiments, the model selects the later-stage answer in over
95\% of disagreements and remains highly confident even when this
preference leads to an incorrect decision
(Section~\pcref{sec:motivation-self-selection}).
This behavior is related to previously documented length and verbosity
biases in LLM-based
evaluation~\citep{zheng2023judging, hu2025explaining}.

Together, these observations suggest that:
\begin{center}
\vspace{-\topsep}
\fcolorbox{purple}{white}{\parbox{0.93\linewidth}{%
model-specific error bias persists across
\tb{generation}, \tb{reasoning}, and \tb{self-evaluation}, that is, the three
stages of same-model computation that test-time scaling methods rely
on. More compute from the same model is therefore not more evidence.}}
\vspace{-\topsep}
\end{center}

These findings suggest that useful additional computation should provide
evidence that is both \emph{selectively acquired} and
\emph{complementary} to the base model's existing reasoning.

\tb{(1)~When to seek additional evidence.}
Prior work shows that intermediate predictions and answer-stability
signals can reveal how reasoning evolves and support adaptive
computation~\citep{ICLR2026_8df90a14, fu2024certaindex, dai2025sgrpo},
while extended reasoning is not always
beneficial~\citep{zhou2026whenmorethinkinghurts}.
Motivated by these observations, we use disagreement between an
intermediate prediction and the final prediction as a simple
per-question signal of reasoning instability.

\tb{(2)~Where the evidence should come from.}
For such cases, we seek complementary evidence from heterogeneous LLMs.
Prior ensemble studies show that different models exhibit complementary
strengths and can succeed on different subsets of
problems~\citep{jiang2023llm, wang2024moa}, suggesting that model
heterogeneity can reduce dependence on the particular error tendencies
of the base model.

\tb{(3)~How the evidence should be combined.}
We let these solvers reason independently, without seeing the base model's
candidate predictions, and aggregate their answers directly, consistent with
evidence that voting-based ensembles can be competitive with more elaborate
multi-agent protocols~\citep{li2024moreagents, smit2023should}.

These three principles lead to \ours{} (\oursfull): agreement between the intermediate and final
predictions is returned directly, whereas such reasoning-stage
disagreement activates a
small heterogeneous solver panel for candidate-blind, answer-level
arbitration.

Across seven benchmarks and two base models (Qwen3-4B and Qwen3-8B;
\citealp{yang2025qwen3technicalreport}),
\ours{} outperforms full chain-of-thought in 12 of 14 settings,
including an improvement from 59.0\% to 66.0\% on
MuSiQue~\citep{trivedi2022musique}, while
activating on only ${\sim}$17\% of examples and using up to
47.2\% fewer tokens than a fully-connected multi-agent baseline that
invokes every agent on all questions.  A multi-agent
baseline that routes independent agent outputs through a centralized judge
underperforms direct answer aggregation, e.g., 60.67\% vs.\ 66.00\% on
MuSiQue with Qwen3-4B, highlighting the importance of the aggregation
protocol.

\smallskip\noindent
\tb{Contributions.}
\begin{enumerate}[leftmargin=1.8em, itemsep=2pt]
    \item[\textcolor{purple}{\ding{172}}]
    \tb{Three manifestations of model-specific error bias.}
    We identify systematic error persistence across three stages of
    same-model computation: repeated sampling can reproduce a shared
    wrong answer, reasoning can preserve an initially incorrect
    prediction, and self-evaluation can systematically favor a later and
    more elaborate reasoning outcome
    (Section~\pcref{sec:motivation}).

    \item[\textcolor{purple}{\ding{173}}]
    \tb{Disagreement-triggered independent arbitration.}
    To mitigate model-specific error bias, we propose \ours{}, a
    framework that detects reasoning instability via reasoning-stage
    disagreement and resolves it through candidate-blind, answer-level
    arbitration by heterogeneous external solvers
    (Section~\pcref{sec:method}).

    \item[\textcolor{purple}{\ding{174}}]
    \tb{Comprehensive evaluation.}
    Across seven benchmarks and two base models, \ours{} improves
    over full chain-of-thought in 12 of 14 settings while activating on
    only ${\sim}$17\% of examples and using up to 47.2\% fewer tokens
    than the multi-agent baseline
    (Section~\pcref{sec:experiments}).
\end{enumerate}

\section{Motivation}
\label{sec:motivation}

\subsection{Generation Bias under Repeated Sampling}
\label{sec:motivation-concentration}

Repeated sampling from a reasoning model can produce varied reasoning traces
while sharing the same failure pattern. We quantify this effect on questions
where majority-vote self-consistency produces an incorrect answer. For each
failed question $q$, we denote by $C_q$ the fraction of incorrect samples
assigned to the most frequent wrong answer.

\begin{figure}[t]
    \centering
    \begin{minipage}[t]{0.64\textwidth}
        \centering
        \vspace{0pt}
        \includegraphics[width=\linewidth]{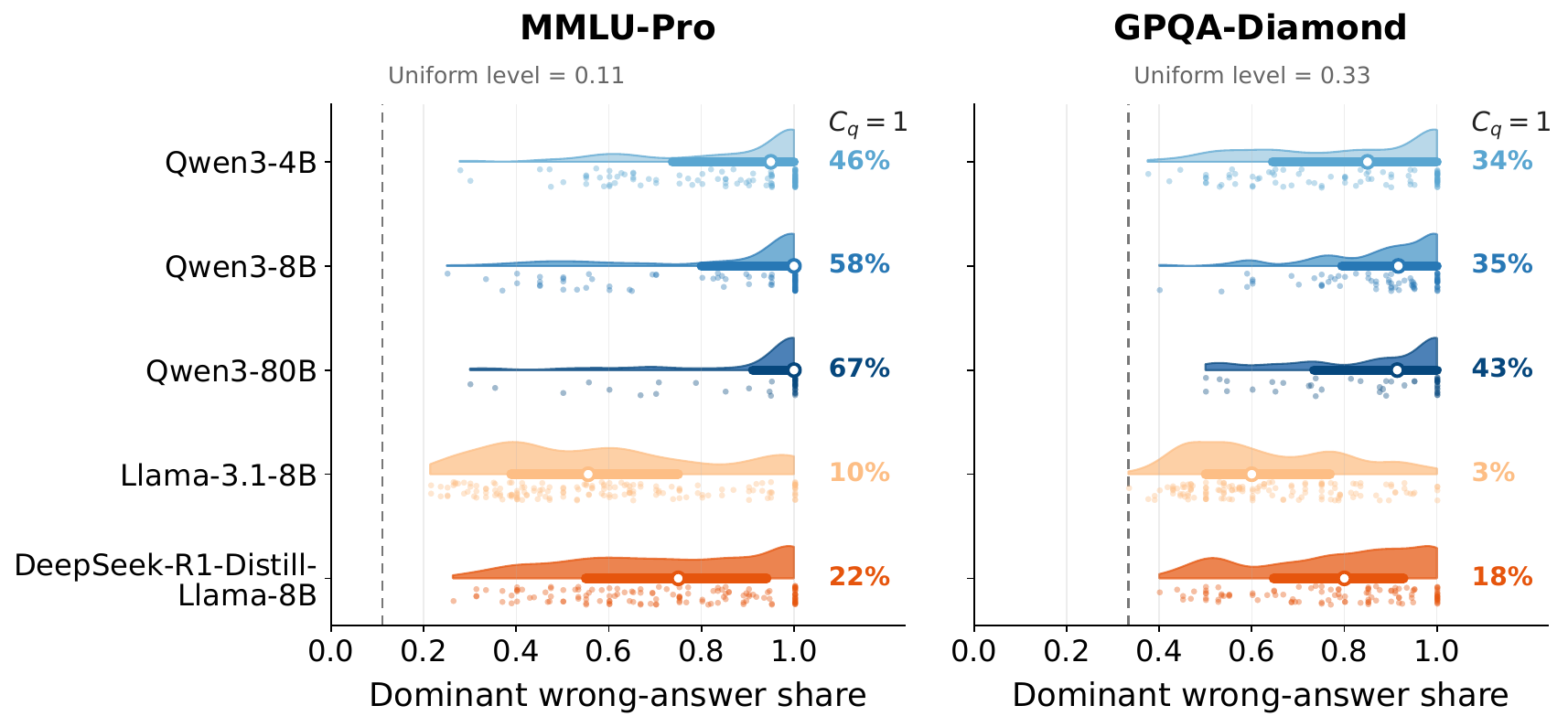}
    \end{minipage}\hfill
    \begin{minipage}[t]{0.34\textwidth}
        \centering
        \vspace{0pt}
        \footnotesize
        \setlength{\tabcolsep}{1.8pt}
        \begin{tabular}{@{}lcc@{}}
            \toprule
            Dataset & \shortstack{Initially\\Wrong} & \shortstack{Error\\Retained} \\
            \midrule
            \multicolumn{3}{l}{\textit{Qwen3-8B}} \\
            MMLU-Pro    & 52.29\% & 26.23\% \\
            MuSR        & 42.00\% & 61.90\% \\
            DetectBench & 43.53\% & 35.54\% \\
            \addlinespace[2pt]
            \multicolumn{3}{l}{\textit{\shortstack[l]{Qwen3-Next-80B-\\A3B-Instruct}}} \\
            MMLU-Pro    & 41.43\% & 32.41\% \\
            MuSR        & 37.50\% & 58.67\% \\
            DetectBench & 39.07\% & 34.86\% \\
            \bottomrule
        \end{tabular}
    \end{minipage}
    \caption{Generation bias and initial-answer persistence bias. \textbf{Left:} Wrong-answer concentration under SC@20 failures. Each point denotes a failed question and reports its dominant wrong-answer share $C_q$; dashed lines mark the uniform-choice baseline. \textbf{Right:} Initial-answer persistence across base-model scales. Error Retained denotes the fraction of initially incorrect answers for which full CoT returns the same incorrect answer.}
    \label{fig:wrong-answer-concentration}
\end{figure}

Figure~\pcref{fig:wrong-answer-concentration} shows the distribution of $C_q$
under SC@20 on MMLU-Pro and GPQA-Diamond. Across models and datasets, the
wrong answers are more concentrated than the uniform-choice baseline. For
many failed questions, all incorrect samples converge on a single answer
($C_q=1$).

Figure~\pcref{fig:wrong-answer-concentration} shows three patterns. First, when
repeated samples from the same model are wrong, they often collapse onto a
single answer rather than spread across alternatives. Second, error
concentration increases with model scale. Third, concentration is higher in
the reasoning-oriented models. \takeaway{Increasing the sampling budget of
the same model may repeatedly reinforce an existing error mode rather than
introduce independent evidence.}

\subsection{Initial-Answer Persistence Bias}
\label{sec:motivation-persistence}

Reasoning is expected to provide additional evidence beyond a model's initial judgment. However, a model may already exhibit a strong answer preference before any explicit reasoning, and the same preference can remain after full reasoning. To examine this initial-answer persistence, we first ask the model to produce a direct answer without reasoning and ask the same model to solve the same question with full chain-of-thought reasoning. We measure whether the final answer retains the initial answer, particularly when the initial answer is incorrect.

The right panel of Figure~\pcref{fig:wrong-answer-concentration} shows that full
CoT often reproduces the model's initially incorrect answer rather than
independently reassessing the question. Same-error persistence is strongest on
MuSR and remains visible across model scales and datasets. \takeaway{Additional reasoning
can elaborate an existing error without supplying the independent
evidence needed to overturn it.}

\subsection{Longer-Reasoning Preference Bias}
\label{sec:motivation-self-selection}

A model's prediction can change as reasoning proceeds. To capture this change, we compare an \emph{intermediate-stage prediction}, obtained at an intermediate reasoning stage where the model is sufficiently confident, with the \emph{final-stage prediction} produced after the model completes its reasoning. When the two disagree, which we call a \emph{reasoning-stage disagreement}, one simple way to resolve the conflict is to let the model select between its own two predictions. This choice may not be neutral: the model may systematically favor the final-stage prediction simply because it follows more reasoning. We examine this bias with Qwen3-8B on reasoning-stage disagreements. Candidates are anonymized and their presentation order is balanced.

\begin{table}[t]
    \centering
    \caption{Self-selection on reasoning-stage disagreements with Qwen3-8B. We report preference for the final-stage prediction, overall accuracy when exactly one candidate is correct, directional accuracy conditioned on which stage is correct, and confidence on incorrect selections. Wrong-choice Conf. is the mean softmax probability assigned to the chosen answer token when the selection is incorrect.}
    \label{tab:self-selection}
    \resizebox{\textwidth}{!}{%
    \begin{tabular}{lrrrrr}
        \toprule
        Dataset & Final-stage Pref. & Self-select Acc.$\uparrow$ & Intermediate$\checkmark$ / Final$\times\uparrow$ & Intermediate$\times$ / Final$\checkmark\uparrow$ & Wrong-choice Conf.$\downarrow$ \\
        \midrule
        MuSR     & 97.7\% & 44.1\% & 0.0\% & 93.8\% & 100.0\% \\
        MMLU-Pro & 95.1\% & 63.4\% & 6.7\% & 96.2\% & 100.0\% \\
        \bottomrule
    \end{tabular}%
    }
\end{table}

Table~\pcref{tab:self-selection} reveals a strong directional asymmetry in self-selection. On reasoning-stage disagreements, the model selects the final-stage prediction in more than 95\% of cases, and its selections are reliable mainly when that prediction is correct. When the intermediate-stage prediction is correct but the final-stage prediction is incorrect, the model rarely selects the earlier prediction. This preference persists with high confidence even when it leads to an incorrect choice. \takeaway{These results reveal a longer-reasoning preference bias: the model tends to trust the later, more elaborate reasoning outcome even when further reasoning has moved away from the correct answer.}

\section{Method}
\label{sec:method}

\paragraph{Framework.}
We propose \ours{} (Disagreement-triggered Independent Arbitration) to
resolve disagreements between predictions produced at two reasoning depths
(Figure~\pcref{fig:method-overview}). Given a question $q$, the base model
yields a \emph{final-stage prediction} after completing its reasoning and,
whenever a confidence criterion is met at an intermediate reasoning stage,
an \emph{intermediate-stage prediction} extracted at that point. For compact
notation, we denote the intermediate-stage and final-stage predictions by
$a_E(q)$ and $a_F(q)$, respectively, and let $m_E(q)\in\{0,1\}$ indicate
whether an intermediate-stage prediction is available. \ours{} takes
these predictions as fixed inputs; the extraction procedure for $a_E(q)$ is
described in Section~\pcref{sec:experimental-setup}.

\begin{figure}[t]
    \centering
    \includegraphics[width=0.95\textwidth]{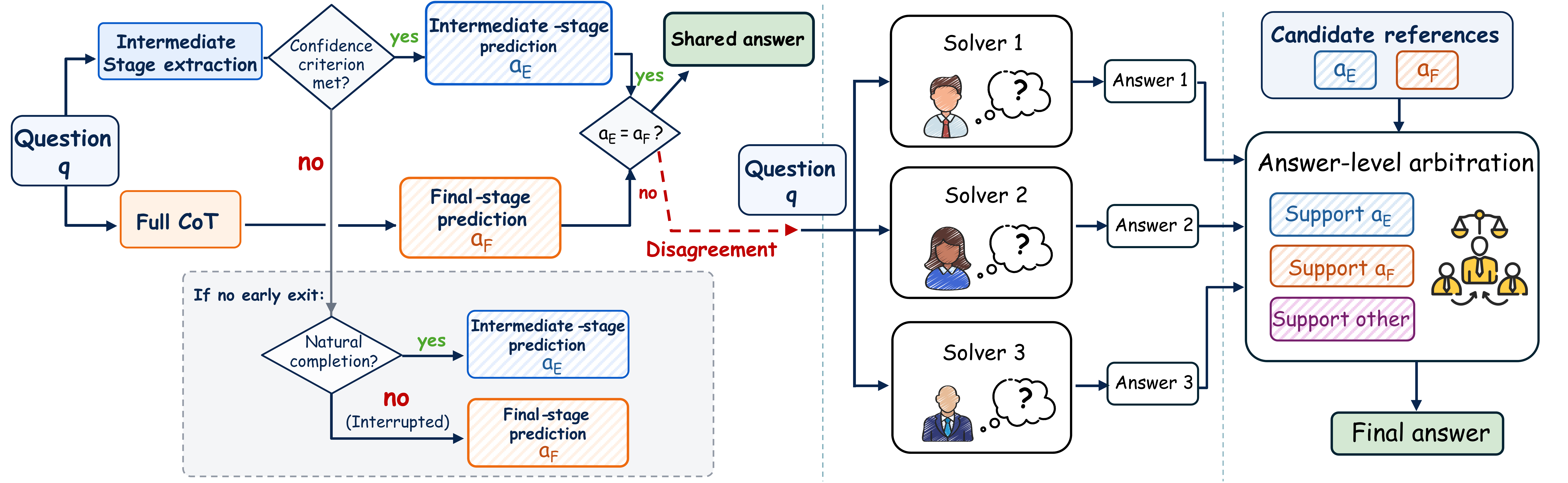}
    \setlength{\abovecaptionskip}{4pt}
    \caption{Overview of \ours. The base model produces an
    intermediate-stage prediction $a_E$ and a final-stage prediction $a_F$.
    Agreement returns the shared answer directly; disagreement activates
    three heterogeneous solvers that solve the original question
    independently behind an information firewall. Answer-level arbitration
    aggregates support for $a_E$, $a_F$, or a third answer.}
    \label{fig:method-overview}
\end{figure}

\paragraph{Disagreement-triggered routing.}
If no intermediate-stage prediction is available ($m_E(q)=0$), we return
$a_F(q)$. If $a_E(q)=a_F(q)$, we return the shared answer. External solvers
are called only on reasoning-stage disagreement, that is, when $a_E(q)$ is
available and differs from $a_F(q)$. We denote this routing decision by
\[
z(q)=
\begin{cases}
1, & \text{if } m_E(q)=1 \text{ and } a_E(q)\ne a_F(q),\\
0, & \text{otherwise}.
\end{cases}
\]

\paragraph{Candidate-blind independent solving.}
When $z(q)=1$, we ask three heterogeneous solvers to solve the original
question. Each solver receives only the question and its answer options and
answers the question independently. No solver sees $a_E(q)$, $a_F(q)$, the
reasoning traces behind them, or another solver's output. Each valid response is
normalized to a task-specific answer representation $v_k$.

\paragraph{Answer-level arbitration.}
Let $\mathcal{V}$ denote the indices of solvers with valid responses. We count
their support for $a_E$ and $a_F$:
\[
n_E=\sum_{k\in\mathcal{V}}\mathbf{1}[v_k=a_E],\qquad
n_F=\sum_{k\in\mathcal{V}}\mathbf{1}[v_k=a_F].
\]
If at least two solvers agree on an answer outside $\{a_E,a_F\}$, we return
that answer. Otherwise, we return $a_E$ when $n_E>n_F$ and $a_F$ when
$n_F>n_E$.

\subsection{Theoretical Analysis}
\label{sec:theory}

We explain why heterogeneous solvers and candidate-blind solving help through
a single quantity, the \emph{truth margin}. Formal statements and proofs are
in Appendix~\pcref{app:proofs}.

\paragraph{Setup.}
For a question with gold answer $y^\star$ and a panel $S$ of $n$ solvers with
average accuracy $\bar p_S$, let $\kappa_S\in[0,1]$ be the fraction of the
panel's expected wrong votes that fall on its single most popular wrong
answer, and define
\begin{equation}
\label{eq:margin}
\Delta_S=\bar p_S-\kappa_S\,(1-\bar p_S).
\end{equation}
$\Delta_S$ is the expected vote share of $y^\star$ minus that of the
strongest wrong answer. The definition encodes an asymmetry: correct votes always coincide on $y^\star$, whereas wrong votes
defeat the truth only if they also coincide on the same wrong answer.

\begin{proposition}[informal; Proposition~\pcref{propf:margin}]
\label{prop:margin}
Under independent voting, plurality is governed by the truth margin
$\Delta_S$, the expected vote share of the gold answer relative to the
strongest wrong alternative. If $\Delta_S>0$, increasing the number of
independent votes concentrates the outcome on the gold answer, exponentially
fast in $n$; if $\Delta_S<0$, additional votes instead reinforce the dominant
error mode.
\end{proposition}

Repeated samples from the same model preserve the answer distribution and
therefore the sign of $\Delta_S$: for a homogeneous panel, $\bar p_S$ is the
model's own accuracy and $\kappa_S$ is exactly the wrong-answer concentration
$C_q$ of Section~\pcref{sec:motivation-concentration}, and neither depends on
the sampling budget. Only a change of solver population can change the sign of
the margin, so the benefit of heterogeneity is not diversity per se but a
reduction of shared error modes; Section~\pcref{sec:ablation-homogeneous}
tests this prediction.

\begin{proposition}[informal; Proposition~\pcref{propf:exposure}]
\label{prop:exposure}
Candidate exposure systematically shifts a solver's answer distribution
toward the exposed candidate. If the exposed candidate is correct, this shift
increases the truth margin. If it is incorrect, the same shift decreases the
margin against that wrong answer and can reverse its sign. Candidate-blind
solving removes this pre-aggregation shift.
\end{proposition}

Exposure amplifies the candidate being shown, not necessarily the truth, so
communication helps only when solvers can reliably distinguish correct from
incorrect shared answers, which is exactly what
Section~\pcref{sec:motivation-self-selection} calls into question. \ours{}
therefore obtains each solver's judgment before revealing any candidate, so
the evidence entering aggregation is not conditioned on the alternatives being
arbitrated.

\begin{corebox}
\centering \ours{} pipeline:
reasoning-stage disagreement $\rightarrow$ heterogeneous external evidence
$\rightarrow$ candidate-blind independent solving $\rightarrow$ direct answer
aggregation.
\end{corebox}

\section{Experiments}
\label{sec:experiments}

\subsection{Experimental Setup}
\label{sec:experimental-setup}

\paragraph{Datasets.}
We evaluate on seven benchmarks across four reasoning domains. MMLU-Pro
\citep{wang2024mmlupro} and GPQA-Diamond \citep{rein2023gpqa} cover
knowledge-intensive reasoning. DetectBench \citep{gu2024detectbench}, MuSiQue
\citep{trivedi2022musique}, and 2WikiMultiHopQA \citep{ho2020twiki} cover
contextual and multi-hop reasoning. We use MATH
\citep{hendrycks2021math}, restricted to problems of difficulty levels 3 to 5, for mathematical reasoning and MuSR
\citep{sprague2024musr} for narrative or soft reasoning. Dataset statistics
and sampling details are provided in Appendix~\pcref{app:dataset-statistics}.

\paragraph{Baselines.}
We compare \ours{} with seven baselines, grouped by the kind of
intervention each applies against model-specific error bias.
\mbox{\tb{\ding{172} Standard reasoning.}}
CoT \citep{wei2022chain} lets the base model complete its reasoning and returns
the final prediction.
\mbox{\tb{\ding{173} Adaptive reasoning.}}
DEER \citep{ICLR2026_8df90a14} and Dynasor \citep{fu2024certaindex} use
intermediate reasoning signals to decide whether continued reasoning is
necessary: DEER relies on confidence at intermediate reasoning states, whereas
Dynasor, in its CoT setting, uses the stability of intermediate predictions.
Like \ours{}, both act on signals exposed before reasoning terminates, but
they modify the computation of the same model rather than seeking evidence
outside it.
\mbox{\tb{\ding{174} Multi-agent inference.}}
We consider two protocols using the same four-agent pool.
\textbf{MAS-Exchange} \citep{wu2023autogen, li2023camel} uses fully-connected
communication: agents first solve the question independently, observe all
peer responses, revise once, and then vote. \textbf{MAS-Judge}
\citep{wang2024moa} removes peer-to-peer communication: agents solve
independently, after which a centralized judge, instantiated by the base
model, observes all responses and selects the final answer.
\mbox{\tb{\ding{175} Learning-based debiasing.}}
We test whether model-specific error bias can be reduced through parameter
updates: SFT trains the base model on its own correct samples and DPO then
prefers correct over incorrect samples of the same question (details in
Appendix~\pcref{app:sft-dpo}).

\paragraph{Implementation.}
We use Qwen3-4B and Qwen3-8B as base models. The external
solvers are DeepSeek-v4-Flash \citep{deepseekai2026deepseekv4},
Llama-3.3-70B-Instruct \citep{llama33}, and GPT-4o-mini \citep{gpt4o-mini}.
Following \citet{ICLR2026_8df90a14}, we extract the intermediate-stage
prediction at reasoning-transition points when the confidence criterion is
met. The final-stage prediction is obtained by letting the same base model
complete its reasoning under matched inference settings.

\paragraph{Metrics.}
We report answer accuracy in percent, and all accuracy differences are
absolute differences. Token Ratio is the total input and output
token usage normalized by the CoT baseline on the same dataset and model.

\paragraph{Evaluation protocol.}
Answers are normalized to option keys before aggregation. The final selected
answer is scored against the benchmark gold option. For free-form outputs, the
final answer is evaluated by Qwen 3.7-Plus (\texttt{qwen/qwen3.7-plus}), an
independent LLM judge, against the benchmark gold answer.

\subsection{Experimental Results}

\begin{table}[t]
    \centering
    \caption{Main results across model scales. For each dataset, adjacent columns report accuracy (\%) and the total input and output token ratio relative to CoT. Average Acc. and Tok. are the corresponding unweighted arithmetic means across the seven datasets. \textbf{Bold} indicates the highest reported accuracy in each column within each base-model block.}
    \label{tab:main-results}
    \setlength{\tabcolsep}{2.8pt}
    \renewcommand{\arraystretch}{1.30}
    \resizebox{\textwidth}{!}{%
\begin{tabular}{cc@{\hspace{3pt}\vrule width 0.2pt\hspace{3pt}}cc cc@{\hspace{3pt}\vrule width 0.2pt\hspace{3pt}}cc cc cc@{\hspace{3pt}\vrule width 0.2pt\hspace{3pt}}cc@{\hspace{3pt}\vrule width 0.2pt\hspace{3pt}}cc@{\hspace{3pt}\vrule width 0.2pt\hspace{3pt}}cc}
        \toprule
        \multirow{3}{*}{Models} & \multirow{3}{*}{Methods}
        & \multicolumn{4}{c@{\hspace{3pt}\vrule width 0.2pt\hspace{3pt}}}{Knowledge-intensive}
        & \multicolumn{6}{c@{\hspace{3pt}\vrule width 0.2pt\hspace{3pt}}}{Contextual \& Multi-hop}
        & \multicolumn{2}{c@{\hspace{3pt}\vrule width 0.2pt\hspace{3pt}}}{Mathematical}
        & \multicolumn{2}{c@{\hspace{3pt}\vrule width 0.2pt\hspace{3pt}}}{Narrative / Soft}
        & \multicolumn{2}{c}{Average} \\
        \cmidrule(lr){3-6}\cmidrule(lr){7-12}\cmidrule(lr){13-14}\cmidrule(lr){15-16}\cmidrule(lr){17-18}
        & & \multicolumn{2}{c}{MMLU-Pro} & \multicolumn{2}{c@{\hspace{3pt}\vrule width 0.2pt\hspace{3pt}}}{GPQA-Diamond}
        & \multicolumn{2}{c}{DetectBench} & \multicolumn{2}{c}{MuSiQue}
        & \multicolumn{2}{c@{\hspace{3pt}\vrule width 0.2pt\hspace{3pt}}}{2Wiki} & \multicolumn{2}{c@{\hspace{3pt}\vrule width 0.2pt\hspace{3pt}}}{MATH}
        & \multicolumn{2}{c@{\hspace{3pt}\vrule width 0.2pt\hspace{3pt}}}{MuSR}
        & \multicolumn{2}{c}{All Datasets} \\
        \cmidrule(lr){3-4}\cmidrule(lr){5-6}\cmidrule(lr){7-8}\cmidrule(lr){9-10}\cmidrule(lr){11-12}\cmidrule(lr){13-14}\cmidrule(lr){15-16}\cmidrule(lr){17-18}
        & & Acc.$\uparrow$ & Tok.$\downarrow$ & Acc.$\uparrow$ & Tok.$\downarrow$
        & Acc.$\uparrow$ & Tok.$\downarrow$ & Acc.$\uparrow$ & Tok.$\downarrow$
        & Acc.$\uparrow$ & Tok.$\downarrow$ & Acc.$\uparrow$ & Tok.$\downarrow$
        & Acc.$\uparrow$ & Tok.$\downarrow$ & Acc.$\uparrow$ & Tok.$\downarrow$ \\
        \midrule
        \multirow{8}{*}{Qwen3-4B}
        & CoT       & 68.00 & 1.00 & 51.01 & 1.00 & 67.74 & 1.00 & 59.00 & 1.00 & 89.50 & 1.00 & 96.67 & 1.00 & 63.50 & 1.00 & 70.77 & 1.00 \\
        & DEER      & 63.14 & 0.86 & 56.06 & 1.60 & 61.65 & 0.67 & 51.67 & 0.84 & 88.00 & 0.29 & 95.33 & 1.18 & 58.50 & 0.68 & 67.76 & 0.87 \\
        & Dynasor   & 63.71 & 1.68 & 51.01 & 1.76 & 65.95 & 1.30 & 55.00 & 1.16 & 90.50 & 1.21 & 95.67 & 2.00 & 63.50 & 1.51 & 69.33 & 1.52 \\
        & MAS-Exchange & \textbf{70.29} & 4.38 & 59.09 & 2.75 & 67.74 & 4.05 & 61.33 & 5.32 & 90.00 & 3.53 & \textbf{97.00} & 2.50 & 62.50 & 4.04 & 72.56 & 3.80 \\
        & MAS-Judge & 66.86 & 3.59 & 58.59 & 2.69 & 66.67 & 3.30 & 60.67 & 4.34 & 90.00 & 3.04 & 96.33 & 2.38 & 61.00 & 3.39 & 71.45 & 3.25 \\
        & SFT       & 66.29 & 1.27 & -- & -- & 67.38 & 1.53 & 57.33 & 1.36 & 88.50 & 1.16 & 96.00 & 1.30 & \textbf{66.50} & 1.36 & -- & -- \\
        & SFT+DPO   & 66.00 & 1.37 & -- & -- & 67.38 & 1.35 & 59.00 & 1.37 & 88.00 & 1.16 & 95.67 & 1.30 & 63.00 & 1.25 & -- & -- \\
        \rowcolor{paleaccent} & \ours{} & \textbf{70.29} & 3.04 & \textbf{59.60} & 2.63 & \textbf{69.18} & 2.80 & \textbf{66.00} & 2.90 & \textbf{91.00} & 2.38 & 96.67 & 2.21 & 64.00 & 2.78 & \textbf{73.82} & 2.68 \\
        \midrule
        \multirow{8}{*}{Qwen3-8B}
        & CoT       & 71.43 & 1.00 & 59.09 & 1.00 & 68.46 & 1.00 & 64.00 & 1.00 & 86.50 & 1.00 & \textbf{98.00} & 1.00 & 60.00 & 1.00 & 72.50 & 1.00 \\
        & DEER      & 68.00 & 0.79 & 60.10 & 1.61 & 64.87 & 0.62 & 56.67 & 0.82 & 87.50 & 0.29 & 96.33 & 0.84 & 61.00 & 0.74 & 70.64 & 0.82 \\
        & Dynasor   & 67.71 & 1.42 & 58.59 & 1.36 & 67.74 & 0.91 & 60.00 & 1.08 & 86.00 & 1.14 & 95.67 & 1.81 & 57.50 & 1.07 & 70.46 & 1.26 \\
        & MAS-Exchange & \textbf{75.14} & 4.27 & \textbf{61.62} & 2.64 & \textbf{69.53} & 3.83 & 62.67 & 5.47 & 88.50 & 3.57 & \textbf{98.00} & 2.22 & 62.50 & 4.11 & 73.99 & 3.73 \\
        & MAS-Judge & 73.14 & 3.48 & 60.61 & 2.62 & 68.10 & 3.34 & 61.67 & 4.43 & \textbf{89.00} & 3.05 & 97.00 & 2.02 & 61.00 & 3.51 & 72.93 & 3.21 \\
        & SFT       & 68.57 & 1.34 & -- & -- & 68.10 & 1.30 & 63.67 & 1.19 & 88.00 & 1.00 & 97.33 & 1.19 & \textbf{63.50} & 1.34 & -- & -- \\
        & SFT+DPO   & 70.57 & 1.34 & -- & -- & 67.74 & 1.36 & 64.00 & 1.04 & 88.00 & 1.12 & --$^{\dagger}$ & -- & \textbf{63.50} & 1.22 & -- & -- \\
        \rowcolor{paleaccent} & \ours{} & \textbf{75.14} & 2.94 & 61.11 & 2.63 & \textbf{69.53} & 2.70 & \textbf{68.67} & 2.89 & \textbf{89.00} & 2.48 & 97.67 & 1.88 & \textbf{63.50} & 2.84 & \textbf{74.95} & 2.62 \\
        \bottomrule
    \end{tabular}%
    }
    \vspace{2pt}

    \begin{minipage}{\textwidth}
    \footnotesize
    $^{\dagger}$The DPO stage was skipped for MATH with Qwen3-8B because
    insufficient valid preference pairs could be constructed to meet our
    minimum training requirement. Only the SFT result is reported.
    \end{minipage}
\end{table}

\obs{172}\quad
\textbf{\ours{} achieves the highest average accuracy at both model scales.}
Across the seven benchmarks, \ours{} improves average accuracy over CoT by
3.05\% and 2.45\% with Qwen3-4B and Qwen3-8B, respectively, and outperforms
CoT in 12 of 14 settings. The two exceptions are on MATH, where CoT accuracy
is already near ceiling and reasoning-stage disagreement leaves little room
for correction.

\obs{173}\quad
\textbf{Detecting unstable reasoning is not sufficient; the source of
corrective evidence matters.} DEER and Dynasor modify or terminate the base
model's reasoning using intermediate signals, whereas \ours{} uses such signals
to request independent evidence from heterogeneous models. \ours{} achieves
higher accuracy than both adaptive-reasoning baselines in all 14 settings,
although at a higher token cost than early exit.

\obs{174}\quad
\textbf{How multiple models are combined matters as much as having multiple
models.} MAS-Exchange introduces cross-agent conditioning before the final
vote, while MAS-Judge preserves independent generation but introduces a
candidate-aware centralized decision. \ours{} avoids both operations and
directly aggregates independently obtained answers. Compared with
MAS-Exchange, it improves average accuracy by 1.26\% and 0.96\% with Qwen3-4B
and Qwen3-8B while lowering the average token ratio by about 30\% at both
scales, and it also exceeds MAS-Judge in average accuracy with fewer tokens;
Section~\pcref{sec:ablation} examines the underlying mechanism.

\obs{175}\quad
\textbf{Learning-based correction does not reliably remove the error modes.}
SFT and SFT+DPO provide limited and inconsistent gains over CoT despite being
trained on correct and incorrect samples of the base model itself. Gains,
where present, are confined to MuSR and, for Qwen3-8B, 2WikiMultiHopQA; in
the remaining settings both variants match or fall below CoT, and adding DPO
after SFT does not consistently help.

\subsection{Ablation Study}
\label{sec:ablation}

\subsubsection{Ablating Independent Solving with Direct Trace Comparison}
\label{sec:ablation-direct-comparison}

We test whether external models can resolve reasoning-stage disagreements by
comparing the candidate reasoning traces rather than solving the question
independently. We use Qwen3-8B as the base model and retain the three default
external models as reviewers on MuSR and MMLU-Pro. Candidates are anonymized
and their presentation order is balanced. Each reviewer independently selects
one candidate, and the final decision follows a two-of-three majority vote.

\paragraph{Full-trace comparison.}
With complete traces, the ensemble selects the final-stage prediction in over
90\% of disagreements on both datasets (Table~\pcref{tab:direct-comparison}),
and rarely recovers a correct intermediate-stage prediction. The final-stage
preference therefore persists even with heterogeneous reviewers.

\paragraph{Equal-budget summaries.}
To examine the role of presentation, Qwen3-8B summarizes each trace with the
same prompt, structure, and approximately matched token budget. This reduces
the final-stage preference and the directional asymmetry, but selection
accuracy improves only slightly on MuSR and decreases on MMLU-Pro.
\takeaway{These results motivate asking external solvers to answer the
question independently rather than judge competing traces.}

\begin{table}[t]
    \centering
    \caption{Direct trace comparison with full traces or equal-budget summaries (Qwen3-8B base model). Final Pref. is the fraction of reasoning-stage disagreements where the reviewer ensemble selects the final-stage prediction. Sel. Acc. is the correct-selection rate when exactly one candidate is correct. The last two columns report the probability of selecting the intermediate-stage prediction when only it is correct, and selecting the final-stage prediction when only it is correct, respectively.}
    \label{tab:direct-comparison}
    \footnotesize
    \setlength{\tabcolsep}{4pt}
    \begin{tabular}{llrrrr}
        \toprule
        Setting & Dataset & Final Pref. & Sel. Acc.$\uparrow$
            & Inter.$\checkmark$ / Final$\times\uparrow$
            & Inter.$\times$ / Final$\checkmark\uparrow$ \\
        \midrule
        \multirow{2}{*}{Full Trace}
            & MuSR     & 93.0\% & 47.1\% & 5.6\%  & 93.8\% \\
            & MMLU-Pro & 91.8\% & 73.2\% & 26.7\% & 100.0\% \\
        \midrule
        \multirow{2}{*}{Equal Budget}
            & MuSR     & 69.8\% & 50.0\% & 33.3\% & 68.8\% \\
            & MMLU-Pro & 61.7\% & 65.0\% & 57.1\% & 69.2\% \\
        \bottomrule
    \end{tabular}
\end{table}

\subsubsection{Homogeneous Qwen3-8B Reviewers}
\label{sec:ablation-homogeneous}

We replace the heterogeneous solvers with three independently sampled Qwen3-8B
solvers. The base model remains Qwen3-8B, and the arbitration rule is unchanged.
Both settings are evaluated on the same activated reasoning-stage disagreements.

Arbitration accuracy decreases on all three datasets, with drops of more than 30\%
on both 2WikiMultiHopQA and MuSiQue
(Table~\pcref{tab:homogeneous-qwen8b-reviewers}). On these two datasets, every
question answered correctly by the homogeneous panel is also answered correctly
by the heterogeneous panel. This pattern is consistent with the model-specific
error bias discussed in Section~\pcref{sec:motivation}. \takeaway{It supports using
heterogeneous solvers to obtain corrective evidence rather than relying on
repeated samples from the base model.}

\subsubsection{Scaling to Stronger Base Models}

We keep the three heterogeneous solvers fixed and evaluate Qwen3 base models at
8B, 30B, and 80B scales\footnote{The 30B and 80B variants are
Qwen3-30B-A3B-Instruct-2507 and Qwen3-Next-80B-A3B-Instruct, respectively.}
on MuSR and MMLU-Pro. \ours{} outperforms both CoT and
DEER in all six settings (Table~\pcref{tab:base-model-scaling}). The gains persist
at 80B without upgrading the external solvers. \takeaway{Together with the
homogeneous-solver ablation in Section~\pcref{sec:ablation-homogeneous}, these results suggest that
heterogeneous models provide useful complementary predictions even as the base
model becomes stronger.}

\begin{table}[t]
    \centering
    \caption{Solver heterogeneity and base-model scaling. \textbf{Left:} Effect of solver heterogeneity with Qwen3-8B as the base model. Arbitration accuracy is measured on activated reasoning-stage disagreements; the homogeneous panel uses three independently sampled Qwen3-8B solvers. \textbf{Right:} Scaling the base model with a fixed heterogeneous solver panel. We report full-dataset accuracy and activation rate, defined as the fraction of examples that trigger arbitration through reasoning-stage disagreement.}
    \label{tab:homogeneous-qwen8b-reviewers}
    \label{tab:base-model-scaling}
    \begin{minipage}[t]{0.39\textwidth}
        \vspace{0pt}
        \centering
        \scriptsize
        \setlength{\tabcolsep}{1.8pt}
        \begin{tabular}{@{}lrrr@{}}
            \toprule
            Dataset & Hetero. & 3 Qwen3-8B & $\Delta$ \\
            \midrule
            MuSR & 55.81\% & 51.16\% & $-4.65\%$ \\
            2WikiMultiHopQA & 53.85\% & 19.23\% & $-34.62\%$ \\
            MuSiQue & 54.72\% & 24.53\% & $-30.19\%$ \\
            \bottomrule
        \end{tabular}
    \end{minipage}
    \hfill
    \begin{minipage}[t]{0.59\textwidth}
        \vspace{0pt}
        \centering
        \scriptsize
        \setlength{\tabcolsep}{3.2pt}
        \begin{tabular}{@{}llrrrr@{}}
            \toprule
            \multirow{2}{*}{Dataset} & \multirow{2}{*}{Base Model} & \multicolumn{3}{c}{Accuracy (\%) $\uparrow$} & \multirow{2}{*}{Activation (\%)} \\
            \cmidrule(lr){3-5}
            & & CoT & DEER & \ours{} & \\
            \midrule
            \multirow{3}{*}{MuSR}
                & Qwen3-8B  & 60.00 & 61.00 & \textbf{63.50} & 21.50 \\
                & Qwen3-30B & 62.00 & 61.50 & \textbf{66.00} & 24.50 \\
                & Qwen3-80B & 66.00 & 67.50 & \textbf{68.50} & 14.50 \\
            \midrule
            \multirow{3}{*}{MMLU-Pro}
                & Qwen3-8B  & 71.43 & 68.00 & \textbf{75.14} & 17.43 \\
                & Qwen3-30B & 75.14 & 72.86 & \textbf{76.00} & 13.43 \\
                & Qwen3-80B & 76.86 & 75.71 & \textbf{78.00} & 9.14 \\
            \bottomrule
        \end{tabular}
    \end{minipage}
\end{table}

Activation rates decrease with model scale, as larger models form more stable
intermediate-stage predictions, yet \ours{} continues to improve accuracy
while invoking external solvers less often.

\section{Related Work}
\label{sec:related-work}

\paragraph{Systematic error bias in LLM reasoning.}
The errors of a language model are not independent across samples, reasoning
steps, or self-evaluation. At the generation level, models commit to an answer
early and then rationalize it \citep{zhang2023language}, larger models
misfire at the commitment step despite retaining the correct answer
\citep{yeom2026hallucination}, and this commitment is skewed by surface
features and prior beliefs
\citep{zheng2024large, jiang2024peek, dasgupta2022language}. Because such
biases are shared across samples, majority voting over repeated samples
\citep{wang2023selfconsistency} can amplify a shared wrong answer rather than
average it out. At the reasoning
level, chain-of-thought explanations can be unfaithful to the factors that
drive the prediction \citep{turpin2023language}, reasoning models form strong
preliminary answer preferences that trigger overthinking
\citep{dang2025internal, chen2025not}, and longer reasoning does not reliably
improve accuracy \citep{zhou2026whenmorethinkinghurts}. At the evaluation
level, LLM judges exhibit position, verbosity, and self-preference biases
\citep{zheng2023judging, wataoka2024self, ye2025justice, hu2025explaining}.
Prior work studies these biases in isolation, largely on multiple-choice or
short chain-of-thought settings; we show that they co-occur within a single
reasoning model and address them at test time by acquiring evidence from
outside the model.

\paragraph{Adaptive reasoning and early exit.}
Test-time compute can be allocated adaptively \citep{snell2025scaling}, either
by constraining reasoning length \citep{han2025token, xu2025chain,
aggarwal2025l1} or by monitoring intermediate answers and stopping once they
stabilize or become confident \citep{fu2024certaindex, ICLR2026_8df90a14,
dai2025sgrpo}. These methods decide \emph{how much} to compute and trust the
intermediate-stage prediction whenever it is confident, which is where the
preliminary-answer bias above resides. \ours{} uses the same kind of signal
to decide \emph{when} to seek external evidence, and retains both the
intermediate-stage and final-stage predictions.

\paragraph{Multi-agent reasoning and its failure modes.}
Multi-agent methods improve reasoning by letting agents exchange and revise
answers \citep{wu2023autogen, li2023camel, liang2024divergent}, by layering
models on earlier outputs \citep{wang2024moa}, or by sampling more agents and
voting \citep{li2024moreagents}. Recent analyses question whether
communication adds independent evidence: agents conform to a majority,
especially when uncertain, and frequently switch from a correct answer to an
incorrect one to match their peers
\citep{zhu2025conformity, cho2025herd, wynn2025talk}; most of the benefit of
debate comes from majority voting rather than from the exchange itself
\citep{smit2023should, choi2026debate}; and many failures of agent systems
trace to inter-agent misalignment and weak verification
\citep{cemri2026multi}. Cost is a further concern, since invoking a full
agent system on every query is expensive, which motivates query-dependent
agent architectures \citep{liu2023dynamic, zhang2025multi}. \ours{} is
designed around these findings: solvers reason independently and never see
the base model's candidates or each other, their answers are aggregated
directly rather than through a candidate-aware judge, which inherits the
judge biases above (Table~\pcref{tab:direct-comparison}), the panel is
heterogeneous because homogeneous solvers share the base model's error modes
(Section~\pcref{sec:ablation-homogeneous}), and it is invoked only on
reasoning-stage disagreements.

\section{Conclusion}
\label{sec:conclusion}

We show that model-specific error bias persists across repeated sampling,
continued reasoning, and self-evaluation, so that more compute from the same
model does not amount to more evidence. \ours{} uses reasoning-stage
disagreement to trigger candidate-blind solving by heterogeneous models and
aggregates their answers directly. Across seven benchmarks, it outperforms CoT
in 12 of 14 settings and achieves higher average accuracy with fewer tokens
than the multi-agent baselines. Ablations show that asking heterogeneous
models to compare the base model's reasoning traces inherits its final-stage
preference, and that replacing the heterogeneous panel with repeated samples
of the base model sharply reduces arbitration accuracy, so independent
solving and solver heterogeneity are both necessary. The gains persist as the
base model scales to 80B while the external solvers remain fixed. Test-time
scaling should therefore consider where additional evidence comes from and
how it is combined, not only how much computation is used.

\section*{Ethics Statement}
We confirm that all authors have read and adhered to the ICLR Code of Ethics. This work complies with the ethical principles outlined therein, and no part of this study violates ethical standards.

\bibliography{references}
\bibliographystyle{iclr2027_conference}

\clearpage
\appendix
\section{Appendix}
\label{app:appendix}

\subsection{Dataset Statistics and Sampling}
\label{app:dataset-statistics}

We evaluate seven benchmark families spanning the four reasoning domains
introduced in the main text. All evaluation subsets are fixed before inference,
and the same canonical inputs are shared across methods and base models.

\paragraph{Knowledge-intensive reasoning.}
MMLU-Pro \citep{wang2024mmlupro} contains 12,032 test questions from 14 broad
knowledge domains. We sample 25 questions from each domain, yielding
a balanced 350-question subset. For GPQA \citep{rein2023gpqa}, we use the full
Diamond split of 198 questions. The split contains
19 Biology, 93 Chemistry, and 86 Physics questions.

\paragraph{Contextual and multi-hop reasoning.}
For DetectBench \citep{gu2024detectbench}, we use 279 questions sampled without
replacement from the 3,928 questions across the official train, dev, test,
test-hard, and test-distract files.
For MuSiQue \citep{trivedi2022musique}, we sample 300 answerable questions from
the 2,417-question MuSiQue-Ans development pool; the resulting subset contains
163 two-hop, 84 three-hop, and 53 four-hop questions. For 2WikiMultiHopQA
\citep{ho2020twiki}, we sample 200 questions from the 12,576-example development
split. This
subset contains 44 bridge-comparison, 48 comparison, 83 compositional, and 25
inference questions.

\paragraph{Mathematical reasoning.}
MATH \citep{hendrycks2021math} covers seven subjects: Algebra, Counting and
Probability, Geometry, Intermediate Algebra, Number Theory, Prealgebra, and
Precalculus. We independently sample 100 questions from each of difficulty
Levels 3, 4, and 5 from the official test split. This produces
three equally sized subsets, or 300 questions in total, and balances the
evaluation across difficulty levels.

\paragraph{Narrative and soft reasoning.}
MuSR \citep{sprague2024musr} comprises murder mysteries, object placements, and
team-allocation tasks. We sample 100 questions from each of the object-placement
and team-allocation task families, yielding 200 evaluation questions in total.

\subsection{SFT and DPO Baselines}
\label{app:sft-dpo}

For each dataset we construct a separate training set, sample the base model
eight times per question, and label each sample as correct or incorrect. SFT
trains the base model for two epochs on the correct samples; DPO then trains
for one epoch on preference pairs in which a correct sample is preferred over
an incorrect sample of the same question. These baselines test whether
repeatedly observed error modes can be learned away, as opposed to bypassed at
inference time using complementary external evidence.

\subsection{Additional Analyses}

\subsubsection{Activation Rates}
\label{app:activation-rates}

Table~\ref{tab:activation-rate} reports the fraction of examples for which
the intermediate-stage and final-stage predictions disagree and selective
arbitration is invoked. Activation is
most frequent on MuSiQue and least frequent on GPQA-Diamond, while the overall
rates remain similar across Qwen3-4B and Qwen3-8B.

\begin{table}[H]
    \centering
    \caption{Activation rates. Each entry is the percentage of evaluation examples for which the intermediate-stage and final-stage predictions disagree and arbitration is invoked.}
    \label{tab:activation-rate}
    \setlength{\tabcolsep}{10pt}
    \begin{tabular}{lcc}
        \toprule
        Dataset & Qwen3-4B & Qwen3-8B \\
        \midrule
        DetectBench & 20.79\% & 17.92\% \\
        MMLU-Pro & 19.43\% & 17.43\% \\
        GPQA-Diamond & 2.02\% & 1.01\% \\
        MATH L3 & 9.00\% & 5.00\% \\
        MATH L4 & 5.00\% & 6.00\% \\
        MATH L5 & 9.00\% & 14.00\% \\
        MATH overall & 7.67\% & 8.33\% \\
        MuSiQue & 33.67\% & 35.33\% \\
        2WikiMultiHopQA & 11.00\% & 13.00\% \\
        MuSR & 21.50\% & 21.50\% \\
        Overall & 17.46\% & 17.13\% \\
        \bottomrule
    \end{tabular}
\end{table}

\subsubsection{Reviewer Count}

We test whether increasing the reviewer ensemble improves arbitration. The base
model is fixed to Qwen3-8B. Starting from DeepSeek-v4-Flash,
Llama-3.3-70B-Instruct, and GPT-4o-mini, we add Phi-4 and
Mistral-Small-3.2-24B to form a five-reviewer ensemble.

\begin{table}[H]
    \setlength{\abovecaptionskip}{0pt}
    \setlength{\belowcaptionskip}{\baselineskip}
    \centering
    \caption{Effect of increasing the reviewer ensemble from three to five models. Results are evaluated only on activated examples, where the intermediate-stage and final-stage predictions disagree. Bold values emphasize arbitration accuracy; the smaller gray text reports the corresponding number of correct decisions. $\Delta$ denotes the change in the number of correct decisions.}
    \label{tab:reviewer-count}
    \footnotesize
    \begin{tabular}{lrrrr}
        \toprule
        Dataset & Activated & 3 Reviewers & 5 Reviewers & $\Delta$ \\
        \midrule
        MuSR & 43 & \textbf{55.81\%} {\color{gray}\scriptsize (24/43)} & \textbf{55.81\%} {\color{gray}\scriptsize (24/43)} & 0 \\
        2WikiMultiHopQA & 26 & \textbf{53.85\%} {\color{gray}\scriptsize (14/26)} & \textbf{53.85\%} {\color{gray}\scriptsize (14/26)} & 0 \\
        DetectBench & 50 & \textbf{48.00\%} {\color{gray}\scriptsize (24/50)} & \textbf{46.00\%} {\color{gray}\scriptsize (23/50)} & $-1$ \\
        \bottomrule
    \end{tabular}
\end{table}

Increasing the ensemble from three to five reviewers provides no improvement:
performance is unchanged on MuSR and 2WikiMultiHopQA and decreases by one
correct arbitration on DetectBench. Additional voters therefore do not help
when their judgments do not add complementary decision signals.

\subsection{Qualitative Comparison: MAS-Exchange Contamination and \ours{} Arbitration}
\label{app:full-mas-contamination}

Figure~\ref{fig:full-mas-contamination} illustrates how unrestricted
cross-agent communication can overturn initially correct independent
solutions. In this DetectBench example, DeepSeek-v4-Flash and
Llama-3.3-70B-Instruct both select the gold answer D in the first round.
After observing the other agents' responses, however, both revise their
answers to C. Although Qwen3-4B moves in the opposite direction and corrects
its answer from C to D, it is outvoted by the three agents that ultimately
select C. The second communication round therefore converts a split panel
containing two correct solutions into an incorrect majority decision.

Figure~\ref{fig:efira-detectbench-case} shows \ours{} on the same
question. The intermediate-stage prediction is the incorrect option C despite
its high confidence, whereas the final-stage prediction reaches the gold answer
D. The three candidate-blind solvers independently return D, D, and C, so
answer-level voting selects the final-stage prediction and preserves the
correct answer. Together, the two examples contrast peer
communication that propagates an incorrect answer with independent solving
that retains complementary evidence until the final vote.

Figure~\ref{fig:efira-mmlu-case} gives a second qualitative example on
MMLU-Pro. The intermediate-stage prediction is E, whereas the final-stage
prediction reaches the gold answer C. The three
independent solvers return C, B, and C, so the two-to-one vote selects C. This
case illustrates how heterogeneous solvers can resolve a reasoning-stage
disagreement even when one solver follows a different error mode.

Figure~\ref{fig:efira-musr-case} gives a complementary example on MuSR. The
intermediate-stage prediction is C, whereas the final-stage prediction reaches
the gold answer B. The independent solvers return B, B, and C, so answer-level
voting preserves the correct final-stage prediction.

Figure~\ref{fig:efira-gpqa-case} shows a GPQA-Diamond example in which the
intermediate-stage prediction is the gold answer C, while the final-stage
prediction changes to A after applying an incorrect
linear-scaling shortcut to the particle-survival probability. The independent
solvers return C, C, and A. The two-to-one vote therefore recovers the
exponential-decay calculation and preserves the correct intermediate-stage
prediction.

Figure~\ref{fig:efira-math-case} gives a mathematical-reasoning example on
MATH Level 3. The intermediate-stage prediction is the correct minimum value
6, whereas the completed reasoning trace concludes 7 after an incorrect
symmetry argument. The independent solvers return 6, 7, and 6, so their
non-unanimous vote selects the intermediate-stage prediction. This
case illustrates how heterogeneous evidence can prevent additional reasoning
from reinforcing a later mathematical error.

\begin{figure}[p]
    \centering
    \includegraphics[width=0.98\textwidth,height=0.86\textheight,keepaspectratio]{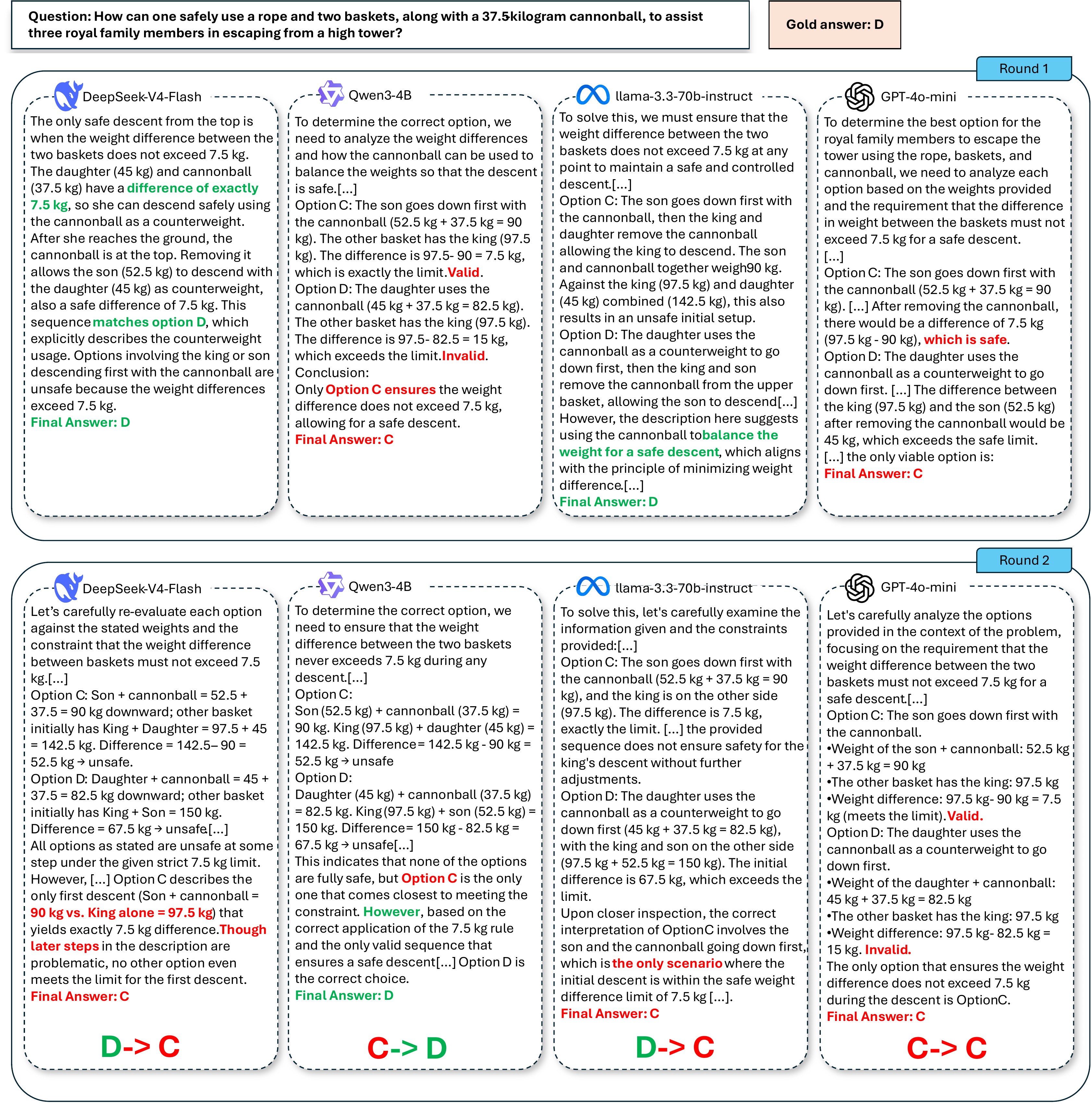}
    \caption{A MAS-Exchange failure case on DetectBench. Round 1 contains two
    correct independent predictions (D) and two incorrect predictions (C).
    After all agents exchange their responses, the two initially correct
    agents switch from D to C, while only Qwen3-4B switches from C to D.
    GPT-4o-mini remains at C, producing a three-to-one incorrect majority in
    Round 2. Green and red text denote correct and incorrect reasoning or
    predictions, respectively; the arrows summarize each agent's answer
    transition.}
    \label{fig:full-mas-contamination}
\end{figure}

\begin{figure}[p]
    \centering
    \includegraphics[
        width=0.98\textwidth,
        height=0.88\textheight,
        keepaspectratio
    ]{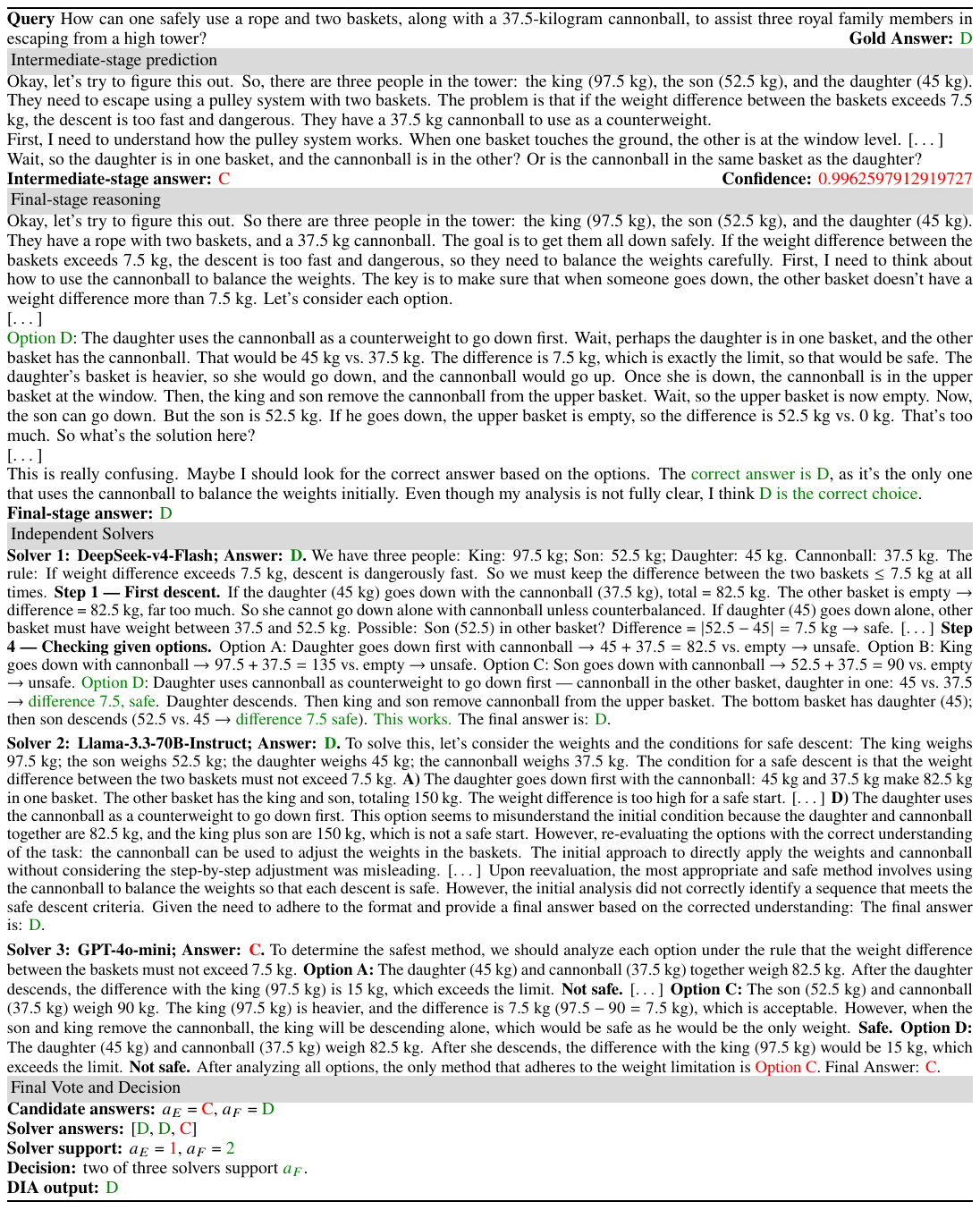}
    \caption{The corresponding \ours{} decision on the same DetectBench
    question. The intermediate-stage prediction is C with high confidence,
    while the final-stage prediction is the gold answer D. Candidate-blind
    independent solvers vote D, D, and C, and answer-level arbitration selects
    the final-stage prediction by a two-to-one majority. Unlike
    MAS-Exchange, the solvers do not observe one another's responses before the
    vote. Green and red text denote correct and incorrect reasoning or
    predictions, respectively.}
    \label{fig:efira-detectbench-case}
\end{figure}

\begin{figure}[p]
    \centering
    \includegraphics[
        width=0.98\textwidth,
        height=0.86\textheight,
        keepaspectratio
    ]{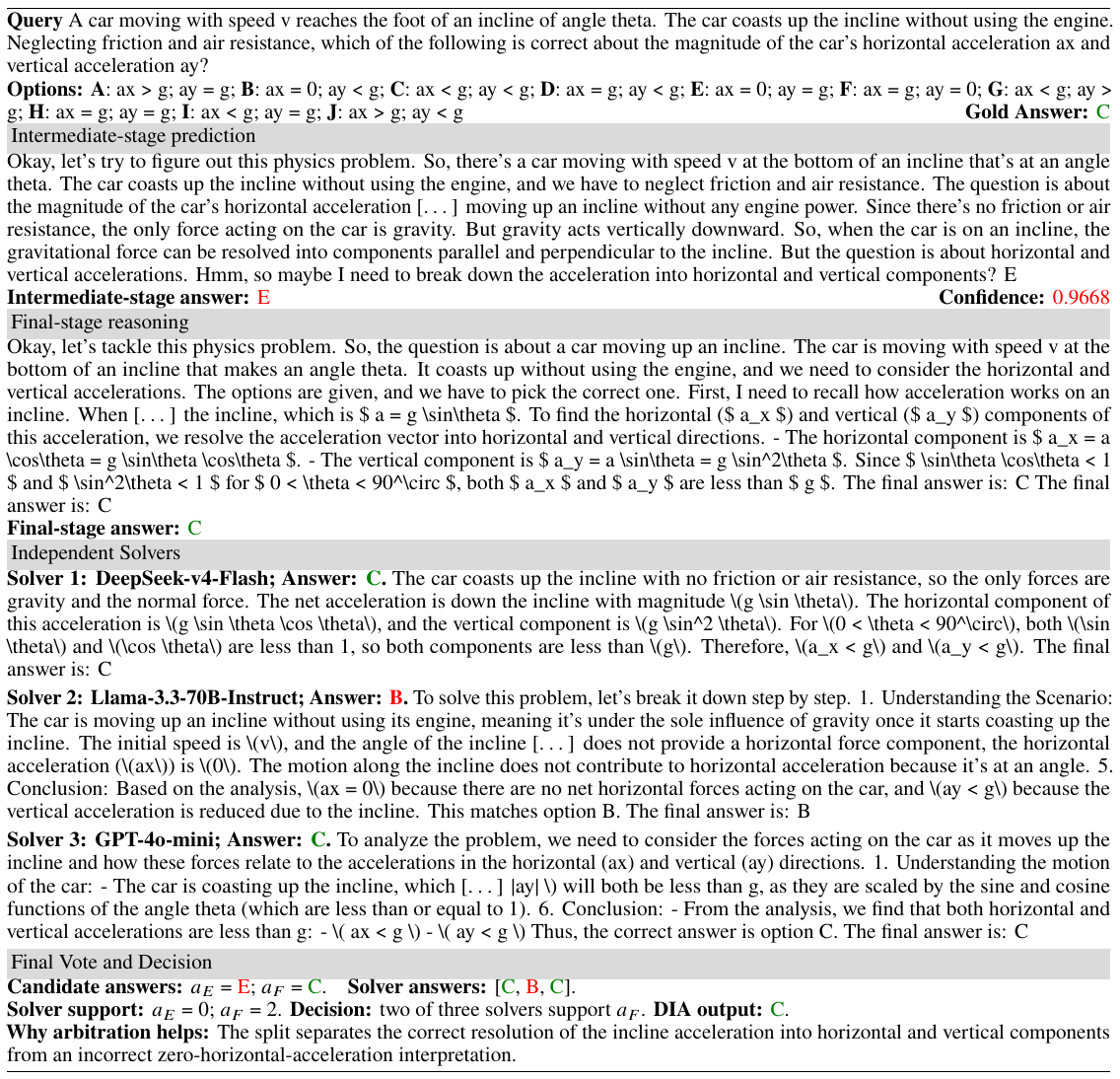}
    \caption{A qualitative \ours{} example on MMLU-Pro (Question 10051;
    Qwen3-4B base model).}
    \label{fig:efira-mmlu-case}
\end{figure}

\begin{figure}[p]
    \centering
    \includegraphics[
        width=0.98\textwidth,
        height=0.86\textheight,
        keepaspectratio
    ]{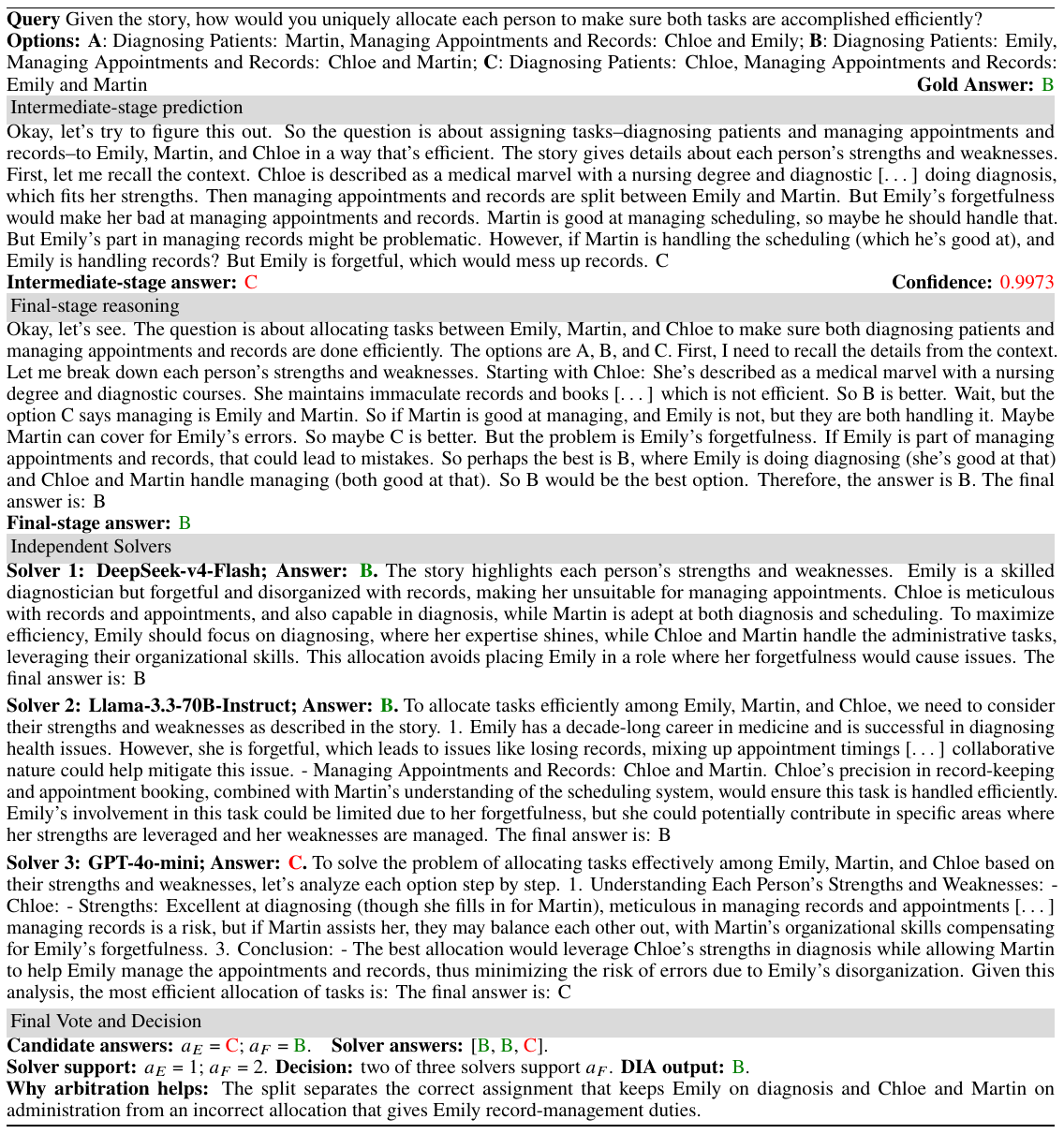}
    \caption{A qualitative \ours{} example on MuSR (Question 109;
    Qwen3-8B base model).}
    \label{fig:efira-musr-case}
\end{figure}

\begin{figure}[tb]
    \centering
    \includegraphics[
        width=0.98\textwidth,
        height=0.86\textheight,
        keepaspectratio
    ]{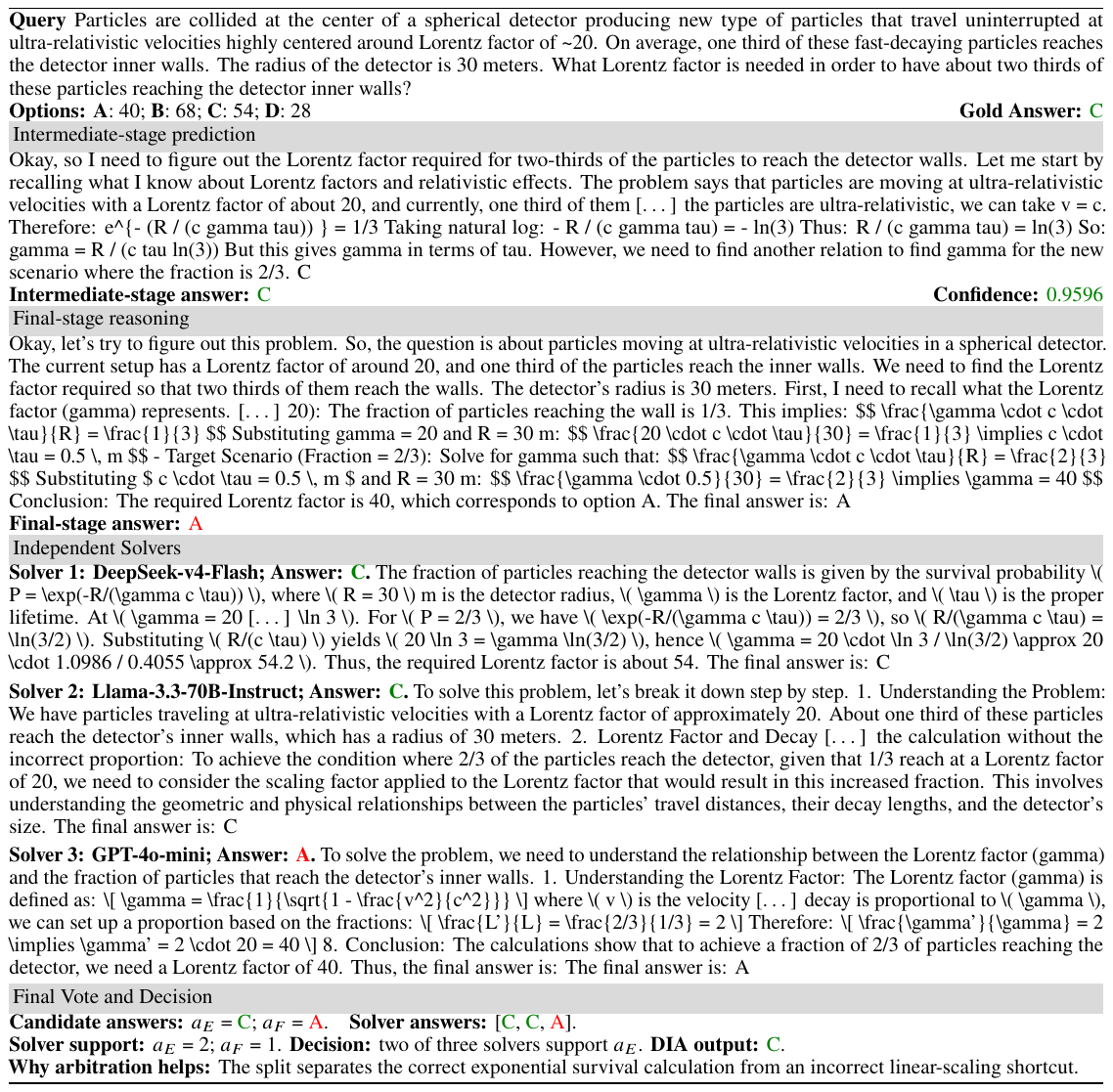}
    \caption{A qualitative \ours{} example on GPQA-Diamond (Question
    \texttt{recmkvk6EFAmqyMxR}; Qwen3-8B base model).}
    \label{fig:efira-gpqa-case}
\end{figure}

\begin{figure}[tb]
    \centering
    \includegraphics[
        width=0.98\textwidth,
        height=0.86\textheight,
        keepaspectratio
    ]{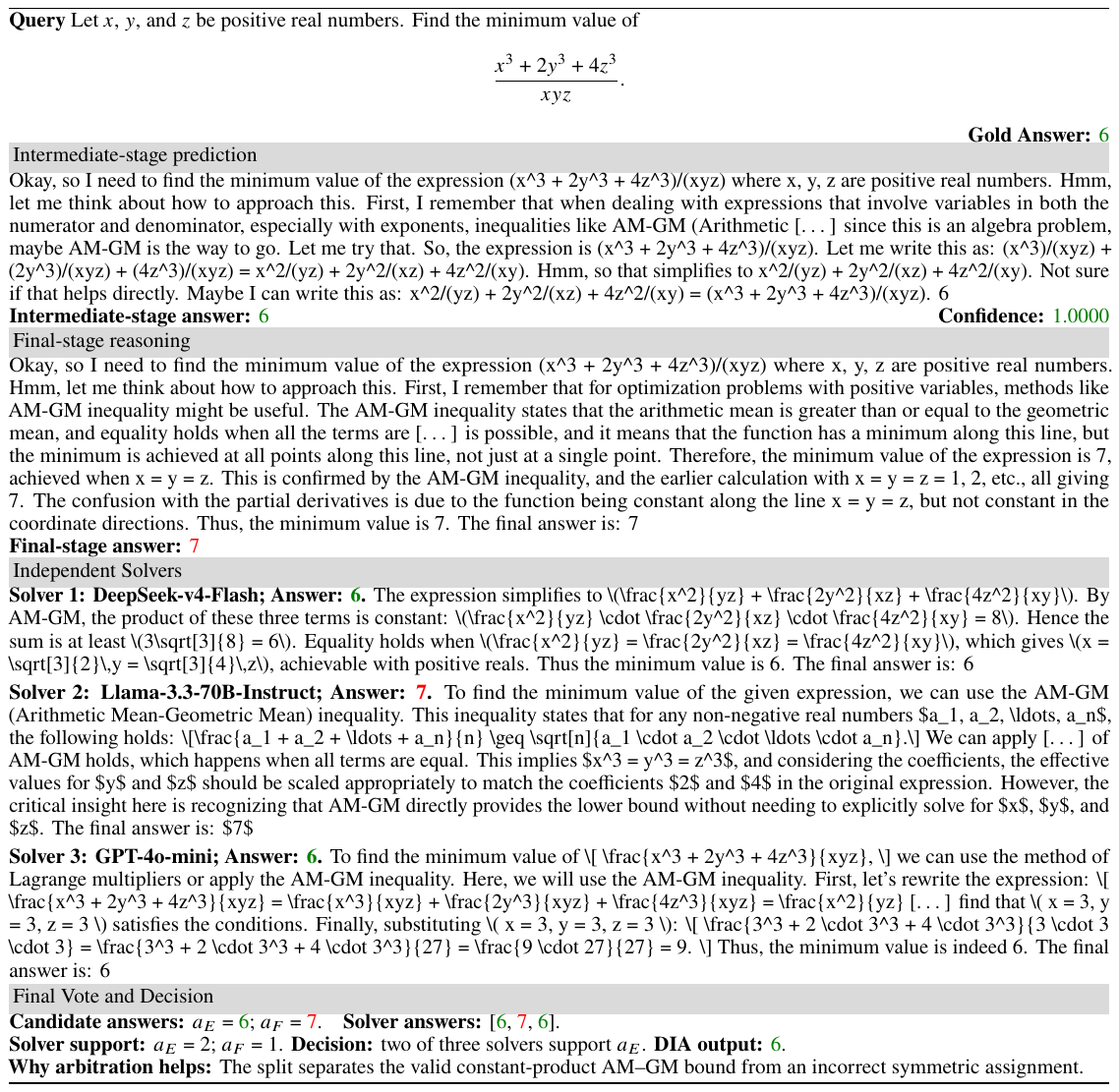}
    \caption{A qualitative \ours{} example on MATH Level 3 (Question
    200; Qwen3-4B base model).}
    \label{fig:efira-math-case}
\end{figure}

\clearpage
\section{Proofs}
\label{app:proofs}

This appendix gives the formal versions of Propositions~\ref{prop:margin}
and~\ref{prop:exposure} and their proofs.

\noindent\textbf{Notation.}
Fix a question $q$ with gold answer $y^\star$ in a finite answer set
$\mathcal{A}$, $|\mathcal{A}|=K$. Solver $i$ answers by sampling
$V_i\sim P_i(\cdot\mid q)$; write $p_i=P_i(y^\star\mid q)$ and
$q_i(w)=P_i(w\mid q)$ for $w\neq y^\star$. For a panel $S$ of $n$ solvers
let $N_a=\sum_{i\in S}\ind{V_i=a}$, $\bar p_S=\frac1n\sum_{i\in S}p_i$,
\[
\Delta_w=\frac{1}{n}\sum_{i\in S}\big[p_i-q_i(w)\big],
\qquad
\kappa_S=\frac{\max_{w\neq y^\star}\sum_{i\in S}q_i(w)}{\sum_{i\in S}(1-p_i)},
\qquad
\Delta_S=\min_{w\neq y^\star}\Delta_w .
\]
Since $\sum_{i\in S}(1-p_i)=n(1-\bar p_S)$, the definition of $\kappa_S$
gives $\max_w\frac1n\sum_i q_i(w)=\kappa_S(1-\bar p_S)$ and hence
$\Delta_S=\bar p_S-\kappa_S(1-\bar p_S)$, which is Eq.~(\ref{eq:margin}).

\subsection{Formal Statement and Proof of Proposition~\ref{prop:margin}}

\begin{propositionf}[Sampling amplifies the sign of the margin]
\label{propf:margin}
Let $V_1,\dots,V_n$ be independent conditional on $q$.
\textup{(i)}~$y^\star$ is the unique maximizer of $\E[N_a]$ over
$a\in\mathcal{A}$ if and only if $\Delta_S>0$.
\textup{(ii)}~If $\Delta_S>0$, then
$\Pr\big[\exists\,w\neq y^\star:\;N_w\ge N_{y^\star}\big]
\le (K-1)\exp\!\big(-n\Delta_S^2/2\big)$.
\textup{(iii)}~If $\Delta_w<0$ for some $w\neq y^\star$, then
$\Pr\big[N_{y^\star}>N_w\big]\le\exp\!\big(-n\Delta_w^2/2\big)$.
\end{propositionf}

\begin{proof}
\emph{(i)} By linearity, $\E[N_a]=\sum_{i\in S}P_i(a\mid q)$. Thus
$\E[N_{y^\star}]>\E[N_w]$ for every $w\neq y^\star$ if and only if
$\Delta_w>0$ for every $w\neq y^\star$, i.e.\ $\Delta_S>0$.

\emph{(ii)} Fix $w\neq y^\star$ and let
$X_i=\ind{V_i=y^\star}-\ind{V_i=w}\in[-1,1]$. The $X_i$ are independent with
$\E[X_i]=p_i-q_i(w)$, so $\sum_{i\in S}X_i=N_{y^\star}-N_w$ has mean
$n\Delta_w\ge n\Delta_S>0$. Hoeffding's inequality for independent variables
with range $2$ gives
\begin{align*}
\Pr\big[N_w\ge N_{y^\star}\big]
&=\Pr\Big[\sum_{i\in S}X_i-\E\sum_{i\in S}X_i\le -n\Delta_w\Big]\\
&\le\exp\!\Big(-\frac{2\,n^2\Delta_w^2}{\sum_{i\in S}2^2}\Big)
=\exp\!\Big(-\frac{n\Delta_w^2}{2}\Big)
\le\exp\!\Big(-\frac{n\Delta_S^2}{2}\Big).
\end{align*}
A union bound over the $K-1$ wrong answers completes the proof.

\emph{(iii)} With the same $X_i$, $\sum_i X_i$ now has mean
$n\Delta_w<0$, and
$\Pr[N_{y^\star}>N_w]=\Pr\big[\sum_i X_i-\E\sum_i X_i>n|\Delta_w|\big]
\le\exp(-n\Delta_w^2/2)$ by the one-sided Hoeffding bound.
\end{proof}

\begin{remark}[Same-model sampling and the sign of the margin]
Repeated samples from one model $m$ are i.i.d.\ from $P_m(\cdot\mid q)$, so
for a homogeneous panel $\bar p_S=p_m$ and $\kappa_S=C_m(q)$, and
$\Delta_S$ does not depend on $n$. Proposition~\ref{propf:margin} then says
that self-consistency converges to $\arg\max_a P_m(a\mid q)$: it raises
accuracy on questions with $\Delta_S>0$ and lowers it on questions with
$\Delta_S<0$, and no sampling budget moves a question between the two sets.
Changing the sign requires changing $\bar p_S$ or $\kappa_S$, i.e.\ changing
the models.
\end{remark}

\begin{remark}[Lower bound on $\kappa_S$]
Since the wrong mass $\sum_{i\in S}(1-p_i)$ is spread over at most $K-1$
answers, $\kappa_S\ge 1/(K-1)$ for every panel. The idealized value
$\kappa_S\approx 1/n$ requires $n\le K-1$, comparable wrong mass across
solvers, and dominant wrong answers on distinct options. For four-option
tasks the floor is $1/3$; for ten-option or free-form tasks it is much
smaller, which is where heterogeneity has the most room to act. Whether this
room translates into observed gains also depends on the activation rate,
solver strength, task difficulty, and answer normalization, so dataset-level
gains should not be read directly from this bound. Adding
solvers changes $n$, $\bar p_S$ and $\kappa_S$ together, so it helps only
insofar as it raises $\Delta_S$.
\end{remark}

\subsection{Formal Statement and Proof of Proposition~\ref{prop:exposure}}

\begin{propositionf}[Candidate exposure is a margin transform]
\label{propf:exposure}
Suppose that, before answering, every solver in $S$ observes the same
candidate $a\in\mathcal{A}$ and outputs $a$ with probability $\gamma_c$ if
$a=y^\star$ and with probability $\gamma_w$ if $a\neq y^\star$, and
otherwise samples from $P_i(\cdot\mid q)$. Let $\Delta'_w$ denote the
resulting margins.
\textup{(i)}~If $a=y^\star$, then
$\Delta'_w=(1-\gamma_c)\Delta_w+\gamma_c\ge\Delta_w$ for all $w$.
\textup{(ii)}~If $a=w_0\neq y^\star$, then
$\Delta'_{w_0}=(1-\gamma_w)\Delta_{w_0}-\gamma_w$ and
$\Delta'_w=(1-\gamma_w)\Delta_w$ for $w\neq w_0$; in particular
$\Delta'_{w_0}<0$ if and only if
$\gamma_w>\Delta_{w_0}/(1+\Delta_{w_0})$.
\textup{(iii)}~If $a$ is drawn from a source with $\Pr[a=y^\star]=r$ and
$\Pr[a=w]=r_w$ for $w\neq y^\star$, and $\gamma_c=\gamma_w=\gamma$, then
$\E[\Delta'_w]=(1-\gamma)\Delta_w+\gamma\,(r-r_w)$.
\end{propositionf}

\begin{proof}
Let $\gamma$ denote $\gamma_c$ if $a=y^\star$ and $\gamma_w$ otherwise.
Solver $i$'s post-exposure answer distribution is the mixture
$P'_i(b\mid q)=(1-\gamma)P_i(b\mid q)+\gamma\,\ind{b=a}$, so
\[
p'_i=(1-\gamma)p_i+\gamma\,\ind{a=y^\star},
\qquad
q'_i(w)=(1-\gamma)q_i(w)+\gamma\,\ind{a=w}.
\]
Subtracting and averaging over $i\in S$,
\[
\Delta'_w=(1-\gamma)\Delta_w+\gamma\big(\ind{a=y^\star}-\ind{a=w}\big).
\]
\emph{(i)} With $a=y^\star$ and $\gamma=\gamma_c$ the bracket equals $1$ for
every $w$, giving $\Delta'_w=(1-\gamma_c)\Delta_w+\gamma_c$. Since
$\Delta_w\le 1$, this is at least $\Delta_w$.

\emph{(ii)} With $a=w_0$ and $\gamma=\gamma_w$ the bracket equals $-1$ for
$w=w_0$ and $0$ otherwise, giving the two stated expressions. Solving
$(1-\gamma_w)\Delta_{w_0}-\gamma_w<0$ for $\gamma_w$ yields
$\gamma_w>\Delta_{w_0}/(1+\Delta_{w_0})$.

\emph{(iii)} Taking expectation over $a$ with $\gamma_c=\gamma_w=\gamma$,
$\E[\ind{a=y^\star}-\ind{a=w}]=r-r_w$, and the claim follows by linearity.
\end{proof}

\begin{remark}[Reading of part (iii)]
The quantity $r-r_w$ is the source's own margin against $w$, so exposure
contracts each solver's margin toward the source's margin by the factor
$\gamma$. It raises the expected margin if and only if the source is more
reliable than the panel. In Section~\ref{sec:motivation-self-selection} the
source is the final-stage prediction with $\gamma\approx 0.95$ and a poor
margin, and in MAS-Exchange the source is a peer of comparable quality, for which
$\E[\Delta'_w]=\Delta_w$. Candidate-blind solving is $\gamma_c=\gamma_w=0$.
\end{remark}

\begin{remark}[Exposure and dependence]
When the candidate $a$ is fixed, the post-exposure votes remain independent
conditional on $(q,a)$; exposure changes the margin, not the independence
structure. Dependence among solvers arises only when the candidate is itself
random and shared, as when agents read one another's first-round answers in
MAS-Exchange. Proposition~\ref{propf:exposure}(iii) shows that in that case the
expected margin of a panel of comparable solvers is unchanged, so any change
in vote accuracy is due to the induced dependence; we do not attempt to bound
this effect and treat the example in
Appendix~\ref{app:full-mas-contamination} as illustrative only.
\end{remark}

\section{AI Use Statement}

LLMs were used as general-purpose assistance tools in a limited capacity during the preparation of this work. Specifically, they assisted with grammar and spelling correction, mathematical derivations, and implementation support for the proposed method and experiments. The research ideas, methodological choices, experimental design, analysis and interpretation of results, and conclusions were developed and determined by the authors. All LLM-assisted derivations, code, and textual revisions were independently reviewed, tested, and verified by the authors before being incorporated into this work. The authors take full responsibility for the scientific content and final presentation of the paper.

\end{document}